\documentclass[letterpaper]{article} 
\usepackage{aaai25}  
\usepackage{times}  
\usepackage{helvet}  
\usepackage{courier}  
\usepackage[hyphens]{url}  
\usepackage{graphicx} 
\usepackage{natbib}  
\usepackage{caption} 
\usepackage{newfloat}
\usepackage{listings}
\DeclareCaptionStyle{ruled}{labelfont=normalfont,labelsep=colon,strut=off} 
\usepackage{graphicx}
\usepackage{amsmath}
\usepackage{amsthm}
\usepackage{booktabs}
\usepackage{todonotes}
\usepackage{tikz}

\usepackage{amssymb}

\newtheorem{definition}{Definition}
\newtheorem{example}{Example}
\newtheorem{proposition}{Proposition}

\newtheorem{theorem}{Theorem}

\newcommand{\jm}[1]{\todo[color=green,inline]{JM: #1}}

\newcommand{\rank}[1]{\text{rank}(#1)}
\newcommand{\Rank}{\text{rank}}

\title{An Extension-Based Argument-Ranking Semantics:\\
Social Rankings in Abstract Argumentation\\ Long Version}

\author {
    Lars Bengel\textsuperscript{\rm 1},
    Giovanni Buraglio\textsuperscript{\rm 2},
    Jan Maly\textsuperscript{\rm 3,2},
    Kenneth Skiba\textsuperscript{\rm 1}
}
\affiliations {
    \textsuperscript{\rm 1}Artificial Intelligence Group, University of Hagen, Hagen, Germany\\
    \textsuperscript{\rm 2}Institute of Logic and Computation, TU Wien, Wien, Austria\\
    \textsuperscript{\rm 3}Institute of Data, Process and Knowledge Management, Vienna University of Economics and Business, Wien, Austria\\
    lars.bengel@fernuni-hagen.de, giovanni.buraglio@tuwien.ac.at, jan.maly@tuwien.ac.at, kenneth.skiba@fernuni-hagen.de
}

\usepackage{bibentry}

\begin{document}

\maketitle

\begin{abstract}

In this paper, we introduce a new family of argument-ranking semantics which can be seen as a refinement of the classification of arguments into skeptically accepted, credulously accepted and rejected. To this end we use so-called social ranking functions which have been developed recently to rank individuals based on their performance in groups. We provide necessary and sufficient conditions for a social ranking function to give rise to an argument-ranking semantics satisfying the desired refinement property. 
\end{abstract}

%

\section{Introduction}

One of the core problems of computational models of argumentation is to classify the quality of 
arguments in the context of a larger discussion. In abstract argumentation, 
this is usually achieved by checking whether an argument is contained in
a set of jointly acceptable arguments, called \emph{extensions}, according to one of several well-established semantics.
While these semantics provide a natural way to rank arguments based on
the larger context of the debate at hand, 
they only allow us to distinguish three types of arguments:
the ones that are \emph{skeptically accepted}, i.\,e. that are contained in every extension;
the ones that are \emph{credulously accepted}, i.\,e. that are contained in at least 
one extension; and the ones that are \emph{rejected}, i.\,e. that are not contained in any extension. 
For this reason, more fine-grained ways of comparing arguments have been
proposed, namely the so called \emph{argument-ranking semantics} \cite{DBLP:journals/jair/CayrolL05,DBLP:conf/sum/AmgoudB13a,AmgoudBDV16,DBLP:conf/aaai/BonzonDKM16,HeyninckRS23}. 
However, generally, such argument-ranking semantics are technically quite
distinct from the extension-based classifications of arguments that 
are more commonly used.

In this paper, we propose a new way of ranking arguments which can be seen 
as a true refinement of the classification in
skeptically, credulously and not accepted arguments. To this end, 
we combine two strands of literature that have emerged recently,
namely \emph{extension-ranking semantics} and \emph{social ranking functions}, in 
a novel way. Intuitively, social ranking functions allow us 
to rank elements based on the quality of sets they are contained in. 
These functions were first introduced in the economics literature 
\cite{DBLP:conf/aldt/MorettiO17}, in order to judge the performance of individuals
based on the success of groups that they were involved in,
and has received significant attention from economists and computer scientists 
\cite{DBLP:conf/ijcai/KhaniMO19,DBLP:conf/ijcai/HaretKMO18,Bernardi19,Suzuki24}.
Unfortunately, semantics that 
only distinguish between sets of arguments that are jointly acceptable and the ones that 
are not do not provide enough information to construct a fine-grained ranking of arguments by applying a social ranking function.
Closer to our needs, \citet{DBLP:conf/ijcai/SkibaRTHK21} recently introduced 
so-called extension-ranking semantics that refine and extend classical argumentation
semantics by providing a partial ranking over sets of arguments.
We employ social ranking functions to this ranking to compare single arguments based on how often an argument can be found in a better extension.
Thus, an argument is preferred to another in the resulting argument-ranking if it contributes to making a larger number of sets acceptable (to a higher degree). Indeed, in this setting social rankings captures a notion of contribution relative to a specific semantics. 

Unfortunately, as mentioned, extension-ranking semantics only provide
a partial ordering, while social ranking 
functions generally take total orders as input. 
We therefore first generalize the theory of social ranking functions
to allow for partial orders, using the so-called
rank of a set.

We then show that, by applying
the right social ranking functions to an extension-ranking 
semantics, we can define argument-ranking semantics that are 
a refinement of the traditional skeptical/credulous acceptance 
of arguments, both in spirit and in a strict technical sense.
More precisely, we show that by applying the \emph{lexicographic 
excellence operator} introduced by \citet{Bernardi19}
to the extension-ranking semantics of
\citet{DBLP:conf/ijcai/SkibaRTHK21} we generate an 
argument ranking such that all skeptically accepted 
arguments are ranked before all credulously accepted 
arguments, which are, in turn, ranked before all rejected 
arguments. More generally, we show which axiomatic properties
are sufficient and necessary for a social ranking operator
to give rise to such a ranking (Section \ref{sec:arg ranking social ranking}). Additionally, we show that
the argument-ranking semantics induced by the lexicographic excellence operator
satisfies these properties and is thus an example of an argument-ranking semantics that satisfies our refinement property.
We conclude by discussing related work (Section \ref{sec:related work}) and then summarizing our results 
and highlighting directions for future research (Section \ref{sec:conclusion}).
Omitted proofs can be found in the appendix.

\section{Preliminaries}\label{sec:preliminaries}

In this section, we introduce the basics of
abstract argumentation literature for our work. 
We will start with the standard model of abstract argumentation,
before introducing argument-ranking and extension-ranking semantics.

\paragraph{Abstract Argumentation Frameworks}
An \emph{abstract argumentation framework} ($AF$) is a directed graph $F=(A,R)$ where $A$ is a (finite) set of \emph{arguments} and $R\subseteq A \times A$ is an \emph{attack relation} among them~\cite{Dung:1995}.
An argument $a$ is said to \emph{attack} an argument $b$ if $(a,b) \in R $. We say that an argument $a$ is \emph{defended by a set} $E \subseteq A$ if every argument $b \in A$ that attacks $a$ is attacked by some $c \in E$. For $a\in A$ we define $a^{-}_{F}=\{b\mid (b,a) \in R \}$ and $a^{+}_{F}=\{b\mid (a, b) \in R\}$ as the sets of arguments attacking $a$ and the sets of arguments that are attacked by $a$ in $F$. For a set of arguments $E \subseteq A$ we extend these definitions to $E^{-}_{F}$ and $E^{+}_{F}$ via $E^{-}_{F} = \bigcup_{a \in E} a^{-}_F$ and $E^{+}_{F} = \bigcup_{a \in E} a^{+}_F$, respectively. If the AF is clear in the context, we will omit the index. 

Most semantics \cite{HOFASemantics} for abstract argumentation are relying on two basic concepts: \emph{conflict-freeness} and \emph{admissibility}.
\begin{definition}
Given $F = (A, R)$, a set $E \subseteq A$ is:
    \emph{conflict-free} iff $\forall a,b \in E$, $(a,b) \not \in R$;
    \emph{admissible} iff it is conflict-free, and every element of $E$ is defended by $E$.
\end{definition}
For an AF $F$ we use $cf(F)$ and $ad(F)$ to denote the sets of conflict-free and admissible sets, respectively.
In order to define the remaining semantics proposed by \citet{Dung:1995} as well as semi-stable semantics~\cite{DBLP:journals/logcom/CaminadaCD12} we make use of the \emph{characteristic function}. 
\begin{definition}
    For an AF $F= (A,R)$ and a set of arguments $E \subseteq A$  the \emph{characteristic function} $\mathcal{F}_{F}(E): 2^{A} \rightarrow 2^{A}$ is defined via:
    $$\mathcal{F}_{F}(E)= \{a \in A | E \text{ defends } a\}$$
    An admissible set $E \subseteq A$ is
 	a \emph{complete} extension ($co$) iff $E = \mathcal{F}_{F}(E)$;
 	a \emph{preferred} extension ($pr$) iff it is a $\subseteq$-maximal complete extension;
 	the unique \emph{grounded} extension ($gr$) iff $E$ is the least fixed point of $\mathcal{F}_{F}$;
 	a \emph{stable} extension ($stb$) iff $E^+_F = A \setminus E$; 
  a \emph{semi-stable} extension ($sst$) iff it is a complete extension, where $E \cup E^+_F$ is $\subseteq$-maximal.
\end{definition}


The sets of extensions of an AF $F$ for these five semantics are denoted as $co(F)$, $pr(F)$, $gr(F)$, $stb(F)$ and $sst(F)$ respectively. Based on these semantics, we can define the status of any argument, namely {\emph{skeptically accepted}} (belonging to each $\sigma$-extension), {\emph{credulously accepted} (belonging to some $\sigma$-extension) and {\emph{rejected} (belonging to no $\sigma$-extension). 
Given an AF $F$ and an extension-based semantics $\sigma$, we use (respectively) $sk_\sigma(F)$, $cred_\sigma(F)$ and $rej_\sigma(F)$ to denote these sets of arguments.

\begin{example}\label{example:af_example}
Consider the AF $F = (A, R)$ depicted as a directed graph in Figure~\ref{fig:example1}, with the nodes corresponding to arguments $A= \{a,b,c,d\}$, and the edges corresponding to attacks $R = \{(a,b), (b,c), (c,d), (d,c)\}$.
\begin{figure}[t] 
 \begin{center}
     \begin{tikzpicture}
        \node (a) [circle, draw, minimum size= 0.65cm] at (0, 0) {$a$};
        \node (b) [circle, draw, minimum size= 0.65cm] at (1, 0) {$b$};
        \node (c) [circle, draw, minimum size= 0.65cm] at (2, 0.0) {$c$};
        \node (d) [circle, draw, minimum size= 0.65cm] at (3, 0.0) {$d$};

        \draw [->] (a) to (b);
        \draw [->] (b) to (c);

        \draw [->, bend right] (c) to (d);
        \draw [->, bend right] (d) to (c);
    \end{tikzpicture}
     \end{center} 
 \caption{Abstract argumentation framework $F_1$ from Example \ref{example:af_example}.}
  \label{fig:example1} 
\end{figure}
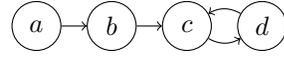
    We see that $F_1$ has three complete extensions $\{a\},\{a,c\}$ and $\{a,d\}$, where only the last two are preferred. In addition, we see that, $a \in sk_{co}(F_1)$, $c, d \in cred_{co}(F_1)$, and $b\in rej_{co}(F_1)$.
\end{example}

An \emph{isomorphism} $\gamma$ between two AFs $F= (A, R)$ and $F'=(A', R')$ is a bijective function $\gamma: F \rightarrow F'$ such that $(a, b) \in R$ iff $(\gamma(a), \gamma(b)) \in R'$ for all $a,b \in A$.

\paragraph{Argument-ranking Semantics}
Instead of reasoning based on the acceptance of sets of arguments, \emph{argument-ranking semantics} (also know as \emph{ranking-based semantics})~\cite{DBLP:conf/sum/AmgoudB13a}
were introduced to focus on the strength of a single argument.
Note that the order returned by an argument-ranking semantics is not necessarily total, i.\,e. not every pair of arguments is comparable.

\begin{definition}
An \emph{argument-ranking semantics} $\rho$ is a function which maps an AF $F=(A,R)$ to a preorder\footnote{A preorder is a (binary) relation that is \emph{reflexive} and \emph{transitive}.} $\succeq_{F}^{\rho}$ on $A$.   
\end{definition}
Intuitively $a \succeq_{F}^{\rho} b$ means that $a$ is at least as strong as $b$ in $F$. 
We define the usual abbreviations as follows;
$a \succ^{\rho}_{F} b$ denotes \emph{strictly stronger}, i.\,e. $a \succeq^{\rho}_{F} b$ and $b \not\succeq^{\rho}_{F} a$. 
Moreover, $a \simeq^{\rho}_{F} b$ denotes \emph{equally strong}, i.\,e. $a \succeq^{\rho}_{F} b$ and $b \succeq^{\rho}_{F} a$. $a \bowtie^\rho_F b$ denotes \emph{incomparability}, meaning that neither $a \succeq^{\rho}_{F} b$ nor $b \succeq^{\rho}_{F} a$.

Traditionally the development of argument-ranking semantics is guided by a principle-based approach \cite{DBLP:journals/flap/TorreV17}. Each principle embodies a different property for argument rankings. We recall one of the most fundamental principle~\cite{DBLP:conf/aaai/BonzonDKM16} as well as a newer one, which is closer to the extension-based reasoning process~\cite{DBLP:conf/comma/BlumelT22}.

\begin{definition}\label{def:ranking_principles}
    An argument-ranking semantics $\rho$ satisfies the respective principle iff for all AFs $F=(A,R)$ and any $a,b \in A$:
    \begin{description}
        \item[Self-Contradicition (\emph{SC}).] Self-attacking arguments should be ranked worse than any other argument. 
        If $(a,a) \notin R$ and $(b,b) \in R$ then $a \succ^\rho_F b$.
        \item[$\sigma$-Compatibility ($\sigma$-\emph{C}).] Credulously accepted arguments should be ranked better than rejected arguments.
        For an extension-based semantics $\sigma$ it holds that if $a \in cred_{\sigma}(F)$ and $b \in rej_\sigma(F)$, then $a \succ^\rho_F b$.
    \end{description}
\end{definition}

Note that principles are not always compatible with each other \cite{DBLP:conf/sum/AmgoudB13a}.

\paragraph{Extension-ranking Semantics}
Extension-ranking semantics defined in~\citet{DBLP:conf/ijcai/SkibaRTHK21} are a generalisation of extension-based semantics. These semantics are used to formalise whether a set $E$ is more plausible to be accepted than another set $E'$. 
\begin{definition}
Let $F= (A,R)$ be an AF. 
    An \emph{extension ranking} on $F$ is a preorder 
    over the powerset of arguments $2^{A}$. An \emph{extension-ranking semantics} $\tau$ is a function that maps each $F$ to an extension ranking $\sqsupseteq^{\tau}_{F}$ on $F$.  
\end{definition}
For an AF $F= (A,R)$, an extension-ranking semantics $\tau$ and two sets $E,E'\subseteq A$ we say $E$ is \emph{at least as plausible to be accepted} as $E'$ with respect to $\tau$ in $F$ if $E \sqsupseteq^{\tau}_{F} E'$. We define the usual abbreviations as follows: $E$ is \emph{strictly more plausible to be accepted} than $E'$ (denoted as $E \sqsupset^{\tau}_{F} E'$) if $E \sqsupseteq^{\tau}_{F} E'$ and not $E' \sqsupseteq^{\tau}_{F} E$; $E$ and $E'$ are \emph{equally as plausible to be accepted} (denoted as $E \equiv^{\tau}_{F} E'$) if $E \sqsupseteq^{\tau}_{F} E'$ and $E' \sqsupseteq^{\tau}_{F} E$; $E$ and $E'$ are \emph{incomparable} (denoted $E \asymp^\tau_{F} E'$) if neither $E \sqsupseteq^{\tau}_{F} E'$ nor $E' \sqsupseteq^{\tau}_{F} E$.

\citet{DBLP:conf/ijcai/SkibaRTHK21} defined a family of approaches to define such extension-ranking semantics. Their semantics are generalisations of the classical extension-based semantics. Using these semantics we can state that a set is ``closer'' to being admissible, than another set. 
Before we define the semantics, we recall the \emph{base relations}, each of them generalises one aspect of extension-based reasoning. 
\begin{definition}[Base Relations~\cite{DBLP:conf/ijcai/SkibaRTHK21}]
    Let $F=(A,R)$ be an AF and $E \subseteq A$ where the function $\mathcal{F}^*_F: \mathcal{P}(A) \rightarrow \mathcal{P}(A)$ is defined as $\mathcal{F}^*_{F}(E) = \bigcup^{\infty}_{i = 1} \mathcal{F}^*_{i,F}(E)$ over the powerset $\mathcal{P}(A)$ of $A$ with 
        $\mathcal{F}^*_{1,F}(E)= E$ and
        $\mathcal{F}^*_{i,F}(E)= \mathcal{F}^*_{i-1,F}(E) \cup (\mathcal{F}_{F}(\mathcal{F}^*_{i-1,F}(E)) \setminus E^-_F)$.    
    Each \emph{base relation} $\alpha \in \{CF,UD, DN, UA\}$ is defined via:
    \begin{itemize}
        \item $CF_F(E) = \{(a,b) \in R | a,b \in E\}$;
        \item $UD_F(E) = E \setminus \mathcal{F}_F(E)$; 
        \item $DN_F(E) = \mathcal{F}^*_F(E) \setminus E$;
        \item $UA_F(E) = \{a \in A \setminus E | \neg \exists b \in E: (b,a) \in R\}$;
    \end{itemize}
\end{definition}
For every base relation, the corresponding $\alpha$ \emph{base extension ranking} $\sqsupseteq^\alpha_F$ for $E, E' \in A$ is given by:
$$ E \sqsupseteq^\alpha_F E' \text{ iff } \alpha_F(E) \subseteq \alpha_F(E')$$
$CF_F(E)$ gives us the conflicts of the set $E$, $UD_F(E)$ the undefended arguments of $E$, $DN_F(E)$ the defended but not included arguments of $E$, where $\mathcal{F}^*_F$ is a generalisation of the characteristic function modelling consistent defence, and $UA_F(E)$ the unattacked arguments of $E$.

By combining these base relations, we denote the extension-ranking semantics.
\begin{definition}\label{def:r-sigma}
    Let $F= (A,R)$ be an AF and $E,E' \subseteq A$. We define:
         \emph{Admissible extension-ranking semantics} $r\text{-}ad$ via $E \sqsupseteq^{r\text{-}ad}_F E'$ iff $E \sqsupset^{CF}_F E'$ or ($E \equiv^{CF}_{F} E'$ and $E \sqsupseteq^{UD}_F E'$).
         \emph{Complete extension-ranking semantics} $r\text{-}co$ via $E \sqsupseteq^{r\text{-}co}_F E'$ iff $E \sqsupset^{r\text{-}ad}_F E'$ or ($E \equiv^{r\text{-}ad}_{F} E'$ and $E \sqsupseteq^{DN}_F E'$).
        \emph{Preferred extension-ranking semantics} $r\text{-}pr$ via $E \sqsupseteq^{r\text{-}pr}_F E'$ iff $E \sqsupset^{r\text{-}ad}_F E'$ or ($E \equiv^{r\text{-}ad}_{F} E'$ and $E' \subseteq E$).
         \emph{Grounded extension-ranking semantics} $r\text{-}gr$ via $E \sqsupseteq^{r\text{-}gr}_F E'$ iff $E \sqsupset^{r\text{-}co}_F E'$ or ($E \equiv^{r\text{-}co}_{F} E'$ and $E \subseteq E'$).
         \emph{Semi-stable extension-ranking semantics} $r\text{-}sst$ via $E \sqsupseteq^{r\text{-}sst}_F E'$ iff $E \sqsupset^{r\text{-}co}_F E'$ or ($E \equiv^{r\text{-}co}_{F} E'$ and $E \sqsupseteq^{UA}_F E'$).
\end{definition}
In words, one set $E$ is at least as plausible to be accepted as $E'$ with respect to the admissible ranking semantics, if $E$ has less conflicts than $E'$ or if they have the same conflicts, then we look at the undefended arguments.
\begin{example}\label{ex:extension rankings}
    Continuing Example \ref{example:af_example}. Comparing sets $E_1= \{c,d\}$ and $E_2 = \{a,c,d\}$ with the admissible ranking semantics, we see $E$ and $E'$ have the same conflicts $(c,d)$ and $(d,c)$, but $E_2$ defends argument $c$ from $b$, so $E_2 \sqsupset^{r\text{-}ad}_{F_1} E_1$, $E_2$ is closer to be an admissible set, then $E_1$.
\end{example}

Extension-ranking semantics also follow a principle-based approach.
Before we recall the principles defined in~\citet{DBLP:conf/ijcai/SkibaRTHK21}, we need to introduce the notion of most plausible sets, i.\,e. sets for which we cannot find any other sets ranked strictly better.
\begin{definition}[Most plausible sets]
Let $F=(A,R)$ be an AF, $E,E'\subseteq A$ two sets of arguments and $\tau$ an extension-ranking semantics. 
    We denote by $max_{\tau}(F)$ the maximal (or \emph{most plausible}) elements of the extension ranking $\sqsupseteq^{\tau}_{F}$, i.\,e. $max_{\tau}(F) = \{E \subseteq A \mid \nexists E' \subseteq A \text{ with } E' \sqsupset^{\tau}_{F} E\}$.
\end{definition}
The principle $\sigma$\emph{-generalisation} states, that the most plausible sets should coincide with the $\sigma$-extensions. 
\begin{definition}[$\sigma$\emph{-Gen}]
Let $\sigma$ be an extension-based semantics and $\tau$ an extension-ranking semantics. $\tau$ satisfies
$\sigma$\emph{-soundness} iff for all $AF$: $max_{\tau}(AF) \subseteq \sigma(AF)$.
$\sigma$\emph{-completeness} iff for all $AF$: $max_{\tau}(AF) \supseteq \sigma(AF)$.
$\sigma$\emph{-generalisation} iff $\tau$ satisfies both $\sigma$-soundness and $\sigma$-completeness.
\end{definition}
Additional principles can be found in \citet{DBLP:conf/ijcai/SkibaRTHK21}.

\section{Social Ranking}

Let us now introduce the final piece of our puzzle, \emph{social rankings}.
Let $S$ be a set of arbitrary objects like players of a sports team, employees of a company or arguments in an AF and $\mathcal{P}(S)$ its powerset.
A \emph{social ranking function} $\xi$, as introduced by~\citet{DBLP:conf/aldt/MorettiO17},
maps a preorder $\sqsupseteq$ on $\mathcal{P}(S)$ to a partial order on $S$.
The most prominent social ranking function is the \emph{lexicographic excellence operator} 
(lex-cel), which was first proposed by \citet{Bernardi19}.
It ranks elements based on the best sets they appear in, 
proceeding lexicographically if there are ties.
In order to make this idea formal,
we need a measure of the quality of a set
that allows us to compare any two sets. For this,
we introduce the notion 
of the \emph{rank} of a set.

\begin{definition}
\label{def: rank}
    Let $X \subseteq S$ be a subset of $S$ and $\sqsupseteq$ a preorder on $\mathcal{P}(S)$. Moreover, let 
    $X_1, X_2, \dots, X_k$ be a longest sequence
    such that $X_1 \sqsupset X_2 \sqsupset \dots \sqsupset X_k \sqsupset X$.
    Then, we define the \emph{rank of $X$}, as $\Rank_\sqsupseteq(X) := k+1$. 

    Moreover, for an element $x \in S$, we define 
    \[x_{k,\sqsupseteq} := |\{X \in \mathcal{P}(S) \mid \Rank_\sqsupseteq(X)= k, x \in X\}|,\]
    as the number of rank $k$ subsets that $x$ is contained in.
\end{definition}
With this definition at hand, we can now define the \emph{lex-cel social ranking function}.

\begin{definition}
    Let $x, y \in S$ be two elements of $S$. 
    We define the \emph{lex-cel} ranking $\succeq^{lex\text{-c}el}$ by (i) $x \succ^{lex\text{-}cel}_{\sqsupseteq} y$ if there exists a $k$ such that $x_{i,\sqsupseteq} = y_{i,\sqsupseteq}$ for all $i < k$ and $x_{k,\sqsupseteq} > y_{k,\sqsupseteq}$ and (ii) $x \simeq^{lex\text{-}cel}_{\sqsupseteq} y$ if $x_{i,\sqsupseteq} = y_{i,\sqsupseteq}$ for all $i \in \mathbb{N}$.
\end{definition}
Intuitively, an object $x$ is ranked better than $y$  by the lexicographic excellence operator if $x$ is contained in more highly ranked sets than $y$. 
\begin{example}\label{ex:lex-cel}
    Consider our running example from Example~\ref{example:af_example}
    and we are using the complete extension-ranking semantics.
    Then, we have three sets with rank $1$, namely the complete extensions.
    The argument $a$ is contained in all three sets with rank $1$,
    while $c$ and $d$ are only contained in one such set each. Consequently
    $a \succeq^{\text{lex-cel}} c$ and $a \succeq^{\text{lex-cel}} d$.
    Now, the final admissible sets $\emptyset$ and $\{d\}$ are dominated by all three 
    complete extensions under the complete extension-ranking semantics, 
    but dominate all non-admissible sets. Therefore, they are the only sets with
    rank $2$. It follows that $d \succeq^{\text{lex-cel}} c$ as both are contained
    in the same number of sets with rank $1$, but $d$ is contained in more sets with
    rank $2$.    
\end{example}

Similarly to argument- and extension-ranking semantics, social rankings 
have been studied axiomatically. Let us first introduce an axiom that has been part of a characterization of the lex-cel function under the assumption that 
the ranking over sets is a total preorder \cite{Bernardi19}.
As we generally do not assume the ranking over extensions to be a total 
preorder, the characterisation does not hold in our setting,
but it is straightforward to see that the lex-cel function still satisfies
this axiom.

\begin{definition}[Independence from the worst set]
    Let $\sqsupseteq$ be a preorder on $\mathcal{P}(S)$ and $X,Y \subseteq S$, let 
    \[w = \max_{X \in \mathcal{P}(S)}(\Rank_\sqsupseteq(X))\]
    and assume that $\sqsupseteq^*$ is another preorder on $\mathcal{P}(S)$ 
    for which it holds 
    \begin{itemize}
        \item $\Rank_{\sqsupseteq^*}(X) = \Rank_\sqsupseteq(X)$ for all $X \in \mathcal{P}(S)$
        s.t. $\Rank_\sqsupseteq(X) < w$.
        \item $\Rank_{\sqsupseteq^*}(X) \geq w $  for all $X \in \mathcal{P}(S)$
        s.t. $\Rank_\sqsupseteq(X) = w$.
    \end{itemize}
    Then for any social ranking function that satisfies \emph{Independence from the worst set},
    we must have that $x \succ_{\sqsupseteq} y$ implies $x \succ_{\sqsupseteq*} y$.
\end{definition}
Intuitively, this axiom states that if one element is already strictly worse than another,
and we further subdivide the worst sets, this strict preference remains.
As we will see later, this axiom will be crucial 
for satisfying our desired refinement property.
We introduce a new axiom inspired by
the classical \emph{Pareto-efficiency} concept \cite{moulin2004fair},
that is satisfied by most reasonable social ranking functions.

\begin{definition}[Pareto-efficiency]
    Let $\sqsupseteq$ be a preorder on $\mathcal{P}(S)$ and let $x,y \in S$ be elements such that
    \begin{itemize}
        \item $\Rank_{\sqsupseteq}(Z \cup \{x\})\leq \Rank_{\sqsupseteq}(Z \cup \{y\})$ for all $Z\in \mathcal{P}(S)$ with $x,y\notin Z$;
        \item $\Rank_{\sqsupseteq}(Z \cup \{x\}) < \Rank_{\sqsupseteq}(Z \cup \{y\})$ for at least one $Z\in \mathcal{P}(S)$ with $x,y\notin Z$.
    \end{itemize}
    Then, for any social ranking function $\xi$ that satisfies \emph{Pareto-efficiency}, we must have 
    $x \succ^{\xi}_{\sqsupseteq} y$.
\end{definition}
Furthermore, we establish the novel \emph{Dominating set} axiom which captures the intuition that if there exists a set containing the object $x$ that is ranked better than every set that contains some other object $y$, then $x$ must be ranked better than $y$ by the social ranking function.

\begin{definition}[Dominating set]\label{def:dominatingset}
    Let $\sqsupseteq$ be a preorder on $\mathcal{P}(S)$ and let $x,y \in S$ such that there exists $X \subseteq S$ with $x \in X$ and for all $Y\subseteq S$ with $y \in Y$ then $X \sqsupset Y$. A social ranking function $\xi$ satisfies \emph{Dominating set} iff $x \succ^{\xi}_\sqsupseteq y$.
\end{definition}

Crucially, Independence from the Worst Set and Pareto-efficiency together imply 
Dominating set. 

\begin{theorem}
    Any social ranking function that satisfies Independence from the worst set
    and Pareto-efficiency also satisfies Dominating set.
\end{theorem}

\section{Defining Argument-ranking Semantics via Social Rankings}\label{sec:arg ranking social ranking}
The idea of combining extension-ranking semantics with argument-ranking semantics was briefly discussed by \citet{DBLP:conf/ijcai/SkibaRTHK21}, where, based on a ranking over sets of arguments, a ranking over arguments was defined. In this section, we take a more general view on this approach and define argument-ranking semantics based on an extension-ranking. 

\subsection{The Singleton Approach}
The most immediate way of ranking objects based on a ranking over sets of objects is to restrict the ranking over sets of objects to the singleton sets. The behaviour of these singleton sets then gives us insight into the relationship between the objects. If $\{a\}$ is ranked better than $\{b\}$ then $a$ is also ranked better than $b$ in the restricted ranking.
 \begin{definition}
 \label{def: Soc Rank Arg Rank}
     Let $F= (A,R)$ be an AF and $\tau$ any extension-ranking semantics. For any two arguments $a,b \in A$, the \emph{singleton argument-ranking semantics} $\mathcal{ST}_\tau$ is defined  via $a \succeq^{\mathcal{ST}_\tau}_F b$ iff $\{a\} \sqsupseteq^{\tau}_F \{b\}$.
 \end{definition}
\citet{Bernardi19} have already discussed that a ranking based solely on singleton sets is too simplistic, as it ignores all the information provided by rankings over sets with cardinality larger than one. 
In the context of abstract argumentation, this is also the case.
\begin{example}\label{exp:naive}
    Consider the AF $F_1$ from Example \ref{example:af_example}. We use $r$-$ad$ as the underlying extension-ranking semantics, then since $\{a\}$ and $\{d\}$ are admissible we have $a =^{\mathcal{ST}_{r\text{-}ad}}_{F_1} d$ and both $\{b\}$ and $\{c\}$ are conflict-free and not defended, so $$a =^{\mathcal{ST}_{r\text{-}ad}}_{F_1} d \succ^{\mathcal{ST}_{r\text{-}ad}}_{F_1} b =^{\mathcal{ST}_{r\text{-}ad}}_{F_1} c$$
 \end{example}
 The example shows that $\mathcal{ST}_{r\text{-}ad}$ has a limited expressiveness, since $\mathcal{ST}_{r\text{-}ad}$ has at most three ranks. The first rank contains arguments for which the singleton set is admissible and the lowest rank are all self-attacking arguments, in between are the non-admissible sets, but conflict-free singleton sets. 
 Observe also that this approach does not refine the 
 classical skeptical/credulous acceptance classification, as
in Example~\ref{exp:naive} the credulously accepted
 argument $c$ is ranked the same as the rejected argument $b$.

\subsection{Generalised Social Ranking Argument-ranking Semantics}

In the literature, a number of different social ranking functions that are more complex than the singleton approach can be found~\cite{AlgabaMRS21,Bernardi19,DBLP:conf/ijcai/HaretKMO18,DBLP:conf/ijcai/KhaniMO19}. 
To understand what constitutes a good social ranking function in this context, we define a general argument-ranking semantics using social ranking solutions with respect to an extension ranking.
\begin{definition}
    Let $F= (A,R)$ be an AF and $\xi$ a social ranking function with respect to extension ranking $\tau$. For any $a,b \in A$ we call $\xi_{\tau}$ the \emph{Social ranking argument-ranking semantics} such that:
       $ a \succeq^{\xi_\tau}_F b \text{ iff } a \succeq^\xi_\tau b$
    \end{definition}
In words, an argument $a$ is at least as strong as argument $b$ if the social ranking function $\xi$ applied to the extension ranking $\tau$ returns that $a$ is at least as strong as $b$. 
\begin{example}
    In Example \ref{ex:lex-cel} the \emph{social ranking argument ranking} $\text{lex-cel}_{\text{r-co}}$ was applied to the AF $F_1$ from Example \ref{example:af_example} where lex-cel is used and the underlying extension-ranking is r-co. Thus, the resulting argument ranking is:
    $$a \succ^{\text{lex-cel}_{\text{r-co}}}_{F_1} d \succ^{\text{lex-cel}_{\text{r-co}}}_{F_1} c \succ^{\text{lex-cel}_{\text{r-co}}}_{F_1} b$$
\end{example}
Any social ranking function can be used to rank arguments. \citet{DBLP:conf/ijcai/SkibaRTHK21} have used a variation of the lex-cel social ranking function in their definitions, where an argument $a$ is ranked better than another argument $b$ if we can find a set $E$ containing $a$ which is ranked better than any set containing $b$. 

\begin{definition}[\cite{DBLP:conf/ijcai/SkibaRTHK21}]\label{def:focusing}
    Let $F=(A,R)$ be an AF, $a,b \in A$, and $\tau$ be an extension-ranking semantics. We define an argument-ranking semantics $\succeq^\tau_{F}$ via $a \succeq^\tau_F b$ iff there is a set $E$ with $a \in E$ s.t. for all sets $E'$ with $b \in E'$ we have $E \sqsupseteq^{\tau}_{F} E'$. 
\end{definition}
\begin{example}\label{ex:focusing}
Continuing with Example \ref{example:af_example}. Using $r\text{-}ad$ as the underlying extension-ranking semantics, we see that $\{a,c\}$ and $\{a,d\}$ are admissible sets, hence also among the most plausible sets. Since $r\text{-}ad$ satisfies $ad$-generalisation there cannot be any set containing $b$ ranked strictly better than these two sets. This observation result in the ranking $a \simeq^{r\text{-}ad}_{F_1} c \simeq^{r\text{-}ad}_{F_1} d \succ^{r\text{-}ad}_{F_1} b$. 
Since $\{a,c\}, \{a,d\} \in \sigma(F)$ for $\sigma \in \{co,pr,stb\}$ the ranking is the same for any $r\text{-}\sigma$. Only for $r\text{-}gr$ the induced ranking differs:  $a \succ^{r\text{-}gr}_{F_1} c \simeq^{r\text{-}gr}_{F_1} d \succ^{r\text{-}gr}_{F_1} b$.
\end{example}

The previous examples show that where $\text{lex-cel}_{\text{r-co}}$ can differentiate $a,b,c, \text{and} \, d$, the argument ranking of Definition \ref{def:focusing} under $r$-$co$ does not allow to distinguish among $a$, $c$ and $d$. 
Indeed, lex-cel is more informative than the operator of \citet{DBLP:conf/ijcai/SkibaRTHK21}.

\begin{proposition}
    Let $F=(A,R)$ be an AF, $a,b \in A$ and $\tau$ an extension ranking. If $a \succeq^{\text{lex-cel}_\tau}_F b$, then $a \succeq^\tau_F b$. 
\end{proposition}

In particular, $\text{lex-cel}_{\text{r-co}}$ allows us to distinguish among skeptically and credulously accepted arguments ($a$ is ranked before $c$ and $d$). To capture this, we define a skeptical variation of $\sigma$-Compatibility. Skeptical accepted arguments are part of every $\sigma$-extension, therefore they should be ranked better than any other argument. 
\begin{definition}
    Let $F=(A,R)$ be an AF, $a,b \in A$, and let $\sigma$ be a extension-based semantics. Argument-ranking semantics $\rho$ satisfies $\sigma$-\emph{skeptical-Compatibility} ($\sigma$-sk-C) iff $a \in sk_{\sigma}(F)$ and $b \notin sk_{\sigma}(F)$ then $a \succ^\rho_F b$.  
\end{definition}

Crucially, a well-behaved argument ranking semantics should be able to rank skeptically accepted arguments before all credulously accepted ones, which should be, in turn, ranked before all rejected arguments. This translated to the following refinement property.

\begin{definition}[$\sigma$-Refinement]\label{def:refinement}
    Argument-ranking semantics $\rho$ satisfies $\sigma$-\emph{Refinement} if $\rho$ satisfies $\sigma$-C and $\sigma$-sk-C for extension-based semantics $\sigma$ for all AFs $F$.
\end{definition}

Next, we investigate principles for social ranking based argument-ranking semantics from a general point of view. We are interested in understanding which combinations of axioms for extension-ranking semantics $\tau$ and social ranking functions $\xi$ represent necessary and sufficient conditions for the corresponding social ranking argument-ranking semantics $\xi_\tau$ to satisfy fundamental principles of argument rankings, chiefly among them our desired refinement property. This translates to the following research questions:
\begin{description}
    \item[RQ1] What properties of $\xi$ and $\tau$ are adequate to ensure that $\xi_\tau$ satisfies a specific principle for argument-ranking semantics?
    \item[RQ2] What properties of $\xi_\tau$ are adequate to ensure that $\xi$ satisfies a specific principle for social ranking functions when combined with a certain extension-ranking semantics $\tau$?
\end{description}
Next, we address RQ1 and RQ2 for a selected number of principles for argument ranking semantics.
 
\subsubsection{Sufficient Conditions for Social Ranking Argument-ranking semantics}
We start by considering $\sigma$-\emph{Compatibility}. For this we show that \emph{Independence from the worst set} together with the quite weak condition \emph{Pareto-efficiency}, is sufficient for satisfying $\sigma$-\emph{C}.

\begin{theorem}
    Let $F=(A,R)$ be an argumentation framework, $\tau$ an extension-ranking semantics satisfying $\sigma$-\emph{generalisation} for the extension semantics $\sigma$ and $\xi$ a social ranking function that satisfies \emph{Independence from the worst set} and \emph{Pareto-efficiency}. 
    Then, $\xi_\tau$ satisfies $\sigma$-\emph{C}. 
\end{theorem}
\begin{proof}
Consider first the extension ranking $\sqsupseteq^\sigma$ defined by
$X \sqsupseteq_F^\sigma Y$ if and only if $X \in \sigma(F)$ and $Y \not \in \sigma(F)$.
Furthermore, let $x \in cred_\sigma(F)$ and $y \in rej_\sigma(F)$.
Then, we claim that $x \succ^{\sqsupseteq^\sigma} y$ for any social ranking function
$\xi$ that satisfies Pareto-efficiency: As $x$ is credulously accepted,
there exists a $X \in \sigma(F)$ with $x \in X$ and
as $y$ is rejected, we have $Y \not \in \sigma(F)$ for all $y \in Y$.
It follows that $\Rank_{\sqsupseteq^\sigma}(X \setminus \{x\}) \cup \{x\}) = 1 <
\Rank_{\sqsupseteq^\sigma}((X \setminus \{x\}) \cup \{y\})$.
On the other hand, there can be no $S$ such that $\Rank_{\sqsupseteq^\sigma}(S \cup \{y\}) < 
\Rank_{\sqsupseteq^\sigma}(S \cup \{x\})$
as, due to the fact that $w = \max_{X \subseteq A} (\Rank_{\sqsupseteq_F^\sigma}(X)) = 2$,
this would imply $\Rank_{\sqsupseteq^\sigma}(S \cup \{y\}) = 1$
and therefore $S \cup \{y\} \in \sigma(F)$.

Furthermore, as $\tau$ satisfies $\sigma$-\emph{generalisation}, we know that
$\Rank_{\sqsupseteq_F^\sigma}(X) = 1$ if and only if 
$\Rank_{\sqsupseteq_F^\tau}(X) = 1$. 
Therefore, it follows from Independence from the worst set that
$x \succ^{\xi_{\sqsupseteq_\sigma}}_F y$ implies $x \succ^{\xi_{\tau}}_F y$.
Consequently, $\xi_{\tau}$ satisfies $\sigma$-\emph{C}.
\end{proof}

Next, we show that Independence from the worst set and Pareto-efficiency together 
also imply that every skeptically accepted argument is ranked before any 
argument that is not skeptically accepted.

\begin{theorem}
    Let $F=(A,R)$ be an AF, $\tau$ an extension-ranking semantics satisfying $\sigma$-generalisation for an extension-based semantics $\sigma$, then if social ranking function $\xi$ satisfies \emph{Pareto-efficiency} and \emph{Independence from the worst set} then $\xi_\tau$ satisfies $\sigma$-sk-C.
\end{theorem}
\begin{proof}
     Let $F=(A,R)$ be an AF, $\tau$ an extension-ranking semantics satisfying $\sigma$-generalisation for an extension-based semantics $\sigma$, and $\xi$ a social ranking function satisfying Pareto-efficiency and Independence from the worst set. Since $\sigma$-generalisation is satisfied by $\tau$ we can view $\tau$ as a refinement of the extension-ranking semantics $\tau'$ defined by $X \sqsupseteq^{\tau'}_F Y$ iff $X \in \sigma(F)$ and $Y \notin \sigma(F)$ for $X, Y \subseteq A$.  
     
     Now consider two arguments $a,b \in A$, such that $a \in sk_\sigma(F)$ and $b \notin sk_\sigma(F)$. Assume there exists a $Z \subseteq A\setminus\{a,b\}$ s.t. $\Rank_{\sqsupseteq_F^{\tau'}}(Z \cup \{b\}) < \Rank_{\sqsupseteq_F^{\tau'}}(Z \cup \{a\})$. Since $\tau'$ only has two levels, this implies $Z \cup \{b\} \in max_{\tau'}(F)$ and thus $Z \cup \{b\} \in \sigma(F)$. As 
     $a \in sk_\sigma(F)$, we must have $a \in Z \cup \{b\}$. However, as $a \notin Z$ we know that also $a \notin Z \cup \{b\}$. This is a contradiction and hence such a $Z$ cannot exist.
     
     Since $b \notin sk_\sigma(F)$ we know there it exists a set $\Bar{Y} \subseteq A$ s.t. $\Bar{Y} \in max_{\tau'}(F)$ and $b \notin \Bar{Y}$. Hence, because $a \in sk_{\sigma}(F)$ we know that $(\Bar{Y} \setminus \{a\}) \cup \{b\} \notin max_{\tau'}(F)$. Consequently, Pareto-efficiency implies $a \succ^{\xi_{\tau'}}_F b$. 

     As  $a \succ^{\xi_{\tau'}}_F b$ holds for $\tau'$, and $\tau$ is a refinement of $\tau'$ such that $max_{\tau'}(F) = max_{\tau}(F)$, it follows from Independence for the worst set that the same holds for $\tau$, i.\,e. $a \succ^{\xi_{\tau}}_F b$.
\end{proof}

Observe that Independence from the worst set means that we might have
to ignore most of the information that is available to us. Next we show that, at least for the rank information, this is essentially unavoidable if we want
to satisfy $cf$-\emph{C}. Let us first introduce an axiom that encodes the idea that 
we cannot ignore overwhelming, rank based evidence.  

\begin{definition}[Rank $k$-super majority]
Let $k \in \mathbb{N}$ be a natural number.
Then we say a social ranking function $\xi$ satisfies 
rank $k$-super majority if for all $x$ and $y$ such that 
\begin{multline*}
|\{Z \in \mathcal{P} \mid x,y \not \in Z \land \rank{Z \cup \{x\}} <
\rank{Z \cup \{y\}}\}| >\\
k\cdot|\{Z \in \mathcal{P} \mid x,y \not \in Z \land \rank{Z \cup \{y\}} 
< \rank{Z \cup \{x\}}\}|,
\end{multline*}
we have $x \succeq y$.
\end{definition}
In words, if there are $k$-times as many sets $Z$ s.t. the rank of 
$Z \cup \{x\}$ is strictly better than the rank of $Z \cup \{y\}$,
than the other way round, then $x$ must be (weakly) preferred to $y$.

\begin{proposition}
    Any social ranking function that, together with $r\text{-}cf$,
    satisfies $cf$-\emph{C} but violates rank $k$-super majority for every $k$.
\end{proposition}

Next, consider the axiom \emph{SC}.
Here, we can find a property of social ranking functions that guarantees that $\xi_\tau$ satisfies \emph{SC} under the assumption that $\tau$ satisfies the following principle:

\begin{definition}[Respects Conflicts]
    For AF $F=(A,R)$ and $E,E' \subseteq A$ extension-ranking semantics $\tau$ satisfies \emph{respects conflicts} if $E \in cf(F)$ and $E' \notin cf(F)$, then $E \sqsupset^{\tau}_F E'$.
\end{definition}
To show that $\xi_\tau$ satisfies \emph{SC} we also need the \emph{Dominating set} property from Definition~\ref{def:dominatingset}.
With these two properties we can then show when SC is satisfied.
\begin{theorem}
    For AF $F=(A,R)$ if extension-ranking semantics $\tau$ satisfies respects conflicts and social ranking function $\xi$ satisfies Dominating set, then $\xi_\tau$ satisfies SC.
\end{theorem}

\subsubsection{Necessary Conditions for Social Ranking Argument-ranking semantics}
Let us try to go the other way, that is finding necessary conditions for 
the social ranking functions to satisfy desirable properties. 
First observe it is not possible to formulate any necessary conditions that also hold for any ranking that cannot be realised by any AF, i.\,e., we cannot find an AF that induces this ranking. 
This is because any property of the argument-ranking only restricts the social ranking function on realisable rankings. Therefore, we need to define the following concept.

\begin{definition}
    Let $X$ be a set of arguments and let $\sqsupseteq$ be a preorder on $\mathcal{P}(X)$.
    Then, we say that $\sqsupseteq$ is $\tau$-\emph{realisable} for a extension-ranking 
    semantics $\tau$ if there is an AF $F=(A,R)$ with $A = X$ such that
    $\sqsupseteq^\tau_F = \sqsupseteq$.
\end{definition}

For example, for a set $\{a,b\}$ any preorder containing $\{a,b\} \sqsupset \{a\}$
is not $r\text{-}cf$-realisable. The conflicts in $\{a,b\}$ must be a strict super-set 
of the conflicts in $\{a\}$. On the other hand, the preorder
containing exactly the relations $\{a\} \sqsupseteq \{a,b\}$ and
$\{b\} \sqsupseteq \{a,b\}$ is realised by AF $F= (\{a,b\}, \{(a,b)\})$.

\begin{theorem}\label{th:cfc_dominating}
    Let $\xi$ be a social ranking function such that $\xi_{r\text{-}cf}$ satisfies $cf\text{-}C$. Then, $\xi$ satisfies Dominating set for all $r\text{-}cf$-realisable preorders $\sqsupseteq$.
\end{theorem}
\begin{proof}
      Let $\sqsupseteq$ be a $cf$-realisable preorder and let $F$ be an $AF$
    that realises it. Assume further that there are $x, y \in A$ such that there 
    exists a $X$ with $x\in X$ for which we have $X \sqsupset Y$ for all $Y$
    such that $y \in Y$. 

    As $X$ contains $x$, its set of conflicts must be a strict super-set
    of the conflicts in $\{x\}$. It follows that $\{x\} \sqsupseteq X \sqsupset Y$
    and hence by transitivity also $\{x\} \sqsupset Y$ for all $Y$ such that 
    $y \in Y$. In particular, it follows that $\{x\} \sqsupset \{y\}$.
    By definition, this means  $CF_F(\{x\}) \subset CF_F(\{y\})$,
    which can only hold if $y$ is self-attacking and $x$ is not.
    However, then $x$ is credulously accepted in the under conflict-free
    semantics while $y$ is not. Consequently, it follows from $cf\text{-}C$ that $x \succ y$.
    Hence, dominating set is satisfied.
\end{proof}

It follows that dominating set is a necessary and sufficient condition for a 
social ranking function to satisfy $cf\text{-}C$ when combined with ${r\text{-}cf}$.
A similar result can be found for admissible semantics. 
\begin{theorem}\label{th:adc_dominating}
    Let $\xi$ be a social ranking s.t. $\xi_{r\text{-}ad}$ satisfies $ad$-$C$. Then $\xi$ satisfies \emph{Dominating set} for all $r$-$ad$-realisable preorders $\sqsupseteq$. 
\end{theorem}
\begin{proof}
     Let $\sqsupseteq$ be a $r$-$ad$-realisable preorder and AF $F=(A,R)$ induces $\sqsupseteq$. Assume $x, y \in A$ such that there exists $X \subseteq A$ with $x \in X$ for which we have $X \sqsupset Y$ for all $Y$ such that $y \in Y$.

    Assume that the set $X$ is not admissible.
    That means one of the following two cases must apply:
    \begin{align*}
       (1)\quad CF_F(X)\neq\emptyset \quad \text{ or, } \quad (2)\quad UD_F(X)\neq\emptyset.
    \end{align*}
    To $(1)$:
    Then, there is some attack $(a,b)\in CF_F(X)$ for $a,b\in X$.
    From $X \sqsupset Y$ it follows that $CF_F(X) \subseteq CF_F(Y)$ and thus $(a,b)\in CF_F(Y)$.
    Now, if $y=a$ or $y=b$ it follows that $y\in X$ which directly contradicts our assumption because of $X \equiv Y'$ for $Y'=X$ with $y\in Y'$.
    However, if $y\neq a$ and $y\neq b$ we can construct $Y'=Y\setminus\{a,b\}$.
    Clearly, that means we either have $CF_F(Y')=\emptyset$ which means $Y \sqsupset X$ or we have $CF_F(Y')\neq\emptyset$ which implies $X \asymp Y'$.
    Because of $y\in Y'$ both cases contradict the initial assumption, hence we must have that $CF_F(X)=\emptyset$, i.\,e. the set $X$ is conflict-free.
    
    To $(2)$:
    Then, there exists an argument $a\in UD_F(X)$ which is not defended by $X$.
    Consider now the set $Y'=\{y\}$ for which we either have that $UD_F(Y')=\emptyset$ or $UD_F(Y')=\{y\}$.
    If $UD_F(Y')=\emptyset$, it follows directly that $Y' \sqsupset X$, contradicting our initial assumption.
    On the other hand, for $UD_F(Y')=\{y\}$ we distinguish between two cases:
    \begin{align*}
        (2.1) \quad y=x, \quad \quad (2.2) \quad y\neq x
    \end{align*}
    Clearly, if $x=y$ we contradict our initial assumption because $X\equiv Y''$ for $Y''=X$.
    Consider now the case $y\neq x$.
    That means, we have that $UD_F(X) \asymp UD_F(Y')$ and thus $X \asymp Y'$.
    Therefore, it follows that we must have $UD_F(X)=\emptyset$, i.\,e. $X$ defends all its elements.
    That means $X$ is admissible and thus it follows directly that $x\in cred_{ad}(F)$.

    From $UD_F(X)=\emptyset$ and $X \sqsupset^{UD}_F Y$ for all $Y$ it follows that $UD_F(Y)\neq\emptyset$.
    Since $\sqsupseteq$ satisfies \emph{ad-generalisation} it follows that $Y\notin ad(F)$ for all $Y$ and thus also $y\in rej_{ad}(F)$.
    Consequently, it follows from $ad\text{-}C$ that $x\succ y$. Hence, Dominating set is satisfied.    
\end{proof}

The previous result suggest that we should check if lex-cel satisfies \emph{Pareto-efficiency}.
\begin{theorem}
    lex-cel satisfies \emph{Pareto-efficiency}.
\end{theorem}

A detailed investigation of the lex-cel argument-ranking semantics $lex\text{-}cel_\tau$ with respect to the satisfied principles will be done in future work.

\section{Related Work}\label{sec:related work}
A number of social ranking functions are discussed in the literature. Like the \emph{ordinal Banzhaf relation} (BI) by \citet{DBLP:conf/ijcai/KhaniMO19} or \emph{ceteris paribus majority relation} (CP) by \citet{DBLP:conf/ijcai/HaretKMO18}. However, the corresponding Social ranking argument-ranking semantics $BI_{\tau}$ and $CP_{\tau}$ do not generalise credulous acceptance, because these two argument-ranking semantics with respect to $r\text{-}ad$ do not satisfy the principle SC (the corresponding counter-examples and definitions can be found in the supplementary material). So, self-contradicting arguments are not necessarily the worst ranked arguments. These two social ranking functions are not suitable to rank arguments in the context of abstract argumentation and therefore we do not discuss them further. 

A number of other argument-ranking semantics were introduced in the literature (for an overview see \citet{DBLP:conf/aaai/BonzonDKM16}). However, the only known argument-ranking semantics satisfying $ad$-Compatibility is the \emph{serialisability-based argument-ranking semantics} (ser) by \citet{DBLP:conf/comma/BlumelT22}. The \emph{serialisability-based argument ranking semantics} ranks arguments according to the number of conflicts that need to be resolved to include these arguments in an admissible set. However, this semantics violates $co$-sk-C. 
\begin{example}\label{example:seri}
Let $F_2$ be the AF as depicted in Figure \ref{fig:example:seri}. Then argument $d \in sk_{co}(F_2)$. So, according to $co$-sk-C it should hold that $d \succ_F a$, however this is not the case for $ser$, i.\,e. $a \succ^{ser}_{F_2} d$. Thus $co$-sk-C is violated. 
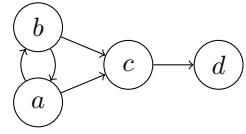
\begin{figure}[t] 
 \begin{center}
     \begin{tikzpicture}
        \node (a) [circle, draw, minimum size= 0.65cm] at (0, 0) {$a$};
        \node (b) [circle, draw, minimum size= 0.65cm] at (0, 1) {$b$};
        \node (c) [circle, draw, minimum size= 0.65cm] at (1.2, 0.5) {$c$};
        \node (d) [circle, draw, minimum size= 0.65cm] at (2.4, 0.5) {$d$};

        \draw [->,bend left] (a) to (b);
        \draw [->, bend left] (b) to (a);

        \draw [->] (b) to (c);
         \draw [->] (a) to (c);
                \draw [->] (c) to (d);
    \end{tikzpicture}
     \end{center} 
 \caption{AF $F_2$ from Example \ref{example:seri}.}
  \label{fig:example:seri} 
\end{figure}
\end{example}
lex-cel$_\tau$ is the only known argument-ranking semantics that satisfies $\sigma$-C and $\sigma$-sk-C and thus satisfies $\sigma$-Refinement for extension-based semantics $\sigma$. 
Thus, lex-cel$_\tau$ is part of none of the equivalence classes of argument-ranking semantics defined by \citet{DBLP:conf/ecsqaru/AmgoudB23}.

\section{Conclusion}\label{sec:conclusion}
In this paper we have combined well-known approaches from abstract argumentation and social ranking functions to define a new family of argument-ranking semantics. 
The resulting semantics are generalisations of the acceptance classifications for abstract argumentation. Thus, the skeptically accepted arguments are ranked before credulously accepted arguments and those are ranked before rejected arguments, and within each of these groupings the arguments are also ranked. All the methods used are off the shelf approaches and already discussed in the literature, showing the connection between social ranking function and argumentation as well as the simplicity of this approach.

The converse problem to social ranking functions are \emph{lifting operators}, i.\,e. given a ranking over objects, we want to construct a ranking over sets of objects. These operators have been discussed for argumentation in the past by \citet{DBLP:conf/comma/YunVCB18} and \citet{DBLP:conf/aaai/maly21}. However, both theses papers do not present a complete picture of lifting operators for abstract argumentation, since they either consider only a subset of sets of arguments (\citet{DBLP:conf/comma/YunVCB18}) or only discuss lifting operators for $ASPIC^+$ (\citet{DBLP:conf/aaai/maly21}). \citet{DBLP:conf/ki/Skiba23} discussed some shortcomings of lifting operators for argumentation frameworks and discussed the need to define lifting operators specifically tailored to abstract argumentation to fully discuss the relationship of argument-ranking semantics, extension-ranking semantics and lifting operators.  

\section*{Acknowledgements}
The research reported here was supported by the Deutsche Forschungsgemeinschaft under grants 375588274 and 506604007, by the
European Research Council (ERC) under the European Union’s Horizon 2020 research
and innovation programme (grant agreement No. 101034440)
and by the Austrian Science Fund (FWF) under grant J4581.

\bibliography{references}
\clearpage

\section*{Appendix}

\setcounter{section}{2}
\section{Social Ranking}

\begin{theorem}
    Any social ranking function that satisfies Independence from the worst set
    and Pareto-efficiency also satisfies Dominating set.
\end{theorem}

\begin{proof}
    Let $\sqsupseteq$ be a preorder on $\mathcal{P}$ and let $x,y$ be elements such that $\exists X^d \subseteq \mathcal{P}$ with $x \in X^d$ such that $\forall X'$ with $y \in X'$ then $X^d \sqsupset X'$.     
    Furthermore, let $w := \Rank_\sqsupseteq(X^d) +1$. We consider the preorder $\sqsupseteq^*$
    that is defined as follows: For any two sets $X,Y \in \mathcal{P}$ we have
    $X \sqsupseteq^* Y$ if and only if $X \sqsupseteq Y$ and either 
    $\Rank_\sqsupseteq(X) < w$ or $\Rank_{\sqsupseteq}(Y) < w$. 
    We claim that  
    \[\max_{X \in \mathcal{P}}(\Rank_{\sqsupseteq^*}(X)) = w.\]
    
    First, to see that $\max_{X \in \mathcal{P}}(\Rank_{\sqsupseteq^*}(X)) \leq w$
    we assume for the sake of a contradiction that there is a set $X$ 
    with $\Rank_{\sqsupseteq*}(X) = w^*> w$. Then, by definition, there is a sequence 
    $X_1 \sqsupseteq^* X_2 \sqsupseteq^* \dots \sqsupseteq^* X_{w^*} \sqsupseteq^* X$.
    As every preference in $\sqsupseteq^*$ is also valid in $\sqsupseteq$,
    the same sequence exists for $\sqsupseteq$, i.\,e. 
    $X_1 \sqsupseteq X_2 \sqsupseteq \dots \sqsupseteq X_{w^*} \sqsupseteq X$.
    However, this means $\Rank_{\sqsupseteq}(X_{w^*}) \geq w^* - 1\geq w$ and
    $\Rank_{\sqsupseteq}(X) \geq w^* > w$, which contradicts $X_{w^*} \sqsupseteq^* X$.
    
    To see that $\max_{X \in \mathcal{P}}(\Rank_{\sqsupseteq^*}(X)) \geq w$
    we first observe that as $\Rank_\sqsupseteq(X^d) = w -1$ there is a sequence
    $X_1 \sqsupseteq X_2 \sqsupseteq \dots \sqsupseteq X_{w -1} \sqsupseteq X$.
    As this sequence is maximal, $\Rank_\sqsupseteq(X_i) < w$ for all elements $X_i$
    of the sequence. Hence the same sequence exists in $\sqsupseteq^*$. 
    Finally, as $X^d$ is a dominating set, we know $X^d \sqsupseteq \{y\}$
    and as $\Rank_\sqsupseteq(X^d) < w$, we also have $X^d \sqsupseteq^* \{y\}$.
    Therefore, $X_1 \sqsupseteq^* X_2 \sqsupseteq^* \dots
    \sqsupseteq^* X_{w -1} \sqsupseteq^* X \sqsupseteq^* \{y\}$ witnesses that  
    $\Rank_{\sqsupseteq^*}(\{y\}) \geq w$.

    Next, we claim that $x \succ^{\sqsupseteq^*} y$ for all social ranking functions that 
    satisfy Pareto-efficiency: By definition, $\Rank_{\sqsupseteq^*}(X^d) = w-1$. Furthermore, we have
    $X^d = (X^d \setminus \{x\}) \cup \{x\} \sqsupseteq^* (X^d \setminus \{x\})
    \cup \{y\}$, and thus  $\Rank_{\sqsupseteq^*}((X^d \setminus \{x\})
    \cup \{y\}) > w-1$. This shows that $\Rank_{\sqsupseteq^*}(X^d) < \Rank_{\sqsupseteq^*}((X^d \setminus \{x\})
    \cup \{y\})$. On the other hand, there can be no $Z$ such that $\Rank_{\sqsupseteq^*}(Z \cup \{y\}) <
    \Rank_{\sqsupseteq^*}(Z \cup \{x\}$): As $Z \cup \{y\}$ is dominated by $X^d$, 
    we know $\Rank_\sqsupseteq(Z \cup \{y\}) \geq w$ and thus $\Rank_{\sqsupseteq}^*(Z \cup \{y\}) \geq w$.
    Thus, the claim follows directly from $\max_{X \in \mathcal{P}}(\Rank_{\sqsupseteq^*}(X)) = w$

    Finally, if $\succeq$ also satisfies Independence from the worst set, 
    if follows that also $x \succ^\sqsupseteq y$, as $\sqsupseteq$ is just a
    refinement of the worst set of $\sqsupseteq^*$.
\end{proof}

%

\section{Defining Argument-ranking Semantics via Social Rankings}\label{sec:arg ranking social ranking}
\begin{proposition}
    Let $F=(A,R)$ be an AF, $a,b \in A$ and $\tau$ an extension ranking. If $a \succeq^{\text{lex-cel}_\tau}_F b$, then $a \succeq^\tau_F b$. 
\end{proposition}
\begin{proof}
 Let $F=(A,R)$ be an AF, $a,b \in A$ and $\tau$ an extension ranking.
Assume $a \succeq^{\text{lex-cel}_\tau}_F b$, then there is an $k$ s.t for all $i < k$ we have $a_{i,\tau} = b_{i,\tau}$ and $a_{k,\tau} \geq b_{k,\tau}$.

If $b_{j,\tau} \neq 0$ for $1\leq j \leq k$, then there is one $Y \subseteq A$ with $rank_\tau(Y) = j$ and $b \in Y$. W.l.o.g. let $j$ be the smallest number s.t. $b_{j,\tau} \neq 0$. Then $Y \sqsupseteq^\tau_F X$ for all $X \subseteq A$ with $a \in X$, therefore $b \succeq^\tau_F a$. Since, $b_{j,\tau} \leq a_{j,\tau}$, there has to be an $X' \subseteq A$ with $rank_\tau(X')=j$ and $a \in X'$ s.t. $X' \equiv^\tau_F Y$, so $a \succeq^\tau_F b$.

If $b_{j,\tau} = 0$ for all $j \in \{1, \dots, k\}$ and $a_{k,\tau} > 0$, then there is at least one $X \subseteq A$ with $a \in X$ and $rank_\tau(X) = k$ s.t. $X \sqsupset^\tau_F Y$ for all $Y \subseteq A$ with $b \in Y$, and therefore $a \succ^\tau_F b$. 
\end{proof}

\begin{proposition}
    Any social ranking function that, together with $r\text{-}cf$,
    satisfies $cf$-\emph{C} and violates rank $k$-super majority for every $k$.
\end{proposition}
\begin{proof}
Let $k$ be an arbitrary natural number, $\ell$ a natural number such that $\ell \geq k$
and $\ell \geq 3$. Furthermore consider an argumentation framework $F$
with the arguments $a, b, c_1, \dots c_\ell$ and the attacks $(b,b)$ and $(c_i,a)$
for all $i \leq \ell$. Then, $a \in cred_{cf}(F)$, as witnessed by the conflict free set
$\{a\}$, but $b \in rej_{cf}(F)$, as it is self-attacking.
It follows from the fact that $a \succ b$, because $\preceq$ satisfies $cf$-\emph{C}.
However, observe that
\begin{multline*}
\{Z \in \mathcal{P} \mid x,y \not \in Z \land \Rank_{r\text{-cf}}(Z \cup \{a\})\} \\<
\Rank_{r\text{-cf}}(Z \cup \{b\}) = \{\emptyset\}
\end{multline*}
while 
\begin{multline*}
\{Z \in \mathcal{P} \mid x,y \not \in Z \\\land \Rank_{r\text{-cf}}(Z \cup \{a\}) <
\Rank_{r\text{-cf}}(Z \cup \{b\})\} \\= \{Z \in \mathcal{P} \mid x,y \not \in Z \land |Z| \geq 2\}.
\end{multline*}
However, then, by our choice of $\ell$ we know 
\[|\{Z \in \mathcal{P} \mid x,y \not \in Z \land |Z| \geq 2\}| > k =
k \cdot |\{\emptyset\}|.\]
It follows that rank $k$-super majority is violated.
\end{proof}

\setcounter{theorem}{3}
\begin{theorem}
    For AF $F=(A,R)$ if extension-ranking semantics $\tau$ satisfies respects conflicts and social ranking function $\xi$ satisfies Dominating set, then $\xi_\tau$ satisfies SC.
\end{theorem}
\begin{proof}
    For AF $F=(A,R)$, let $a,b \in A$, $(b,b) \in R$  and $(a,a) \notin R$, then $\{a\} \in cf(F)$ and for all $E'$ with $b \in E'$ it holds that $E' \notin cf(F)$. Because of respects conflicts we have $\{a\} \sqsupset E'$ and therefore because of Dominating set we have $a \succ^{\xi_\tau}_F b$. 
\end{proof}
\setcounter{theorem}{6}
\begin{theorem}
    lex-cel satisfies \emph{Pareto-efficiency}.
\end{theorem}
\begin{proof}
    First, consider sets $Z_{1},\dots ,Z_{n}\in \mathcal{P}$ for which condition $(2)$ of Pareto-efficiency holds.
    Among these, take those $Z_{1},\dots ,Z_{m}$ (with $m\leq n$) for which $\Rank_{\sqsupseteq}(Z_{j} \cup \{x\})=k$ (with $1\leq j\leq m$) is minimal. 
    At this level in the ranking, we have that $\Rank_{\sqsupseteq}(Z \cup \{x\})=\Rank_{\sqsupseteq}(Z \cup \{y\})$ for each $Z\neq Z_j$. Hence, for every $Z\cup \{x\}$ there is exactly one corresponding set $Z\cup \{y\}$, except for each $Z_j\cup \{x\}$ (because $\Rank_{\sqsupseteq}(Z_j \cup \{y\})>k$). Thus, for each $Z\in \mathcal{P}$ with $x,y\notin Z$:
    \begin{multline*}
        |\{Z \cup \{x\} \in \mathcal{P}\mid  \Rank_{\sqsupseteq}(Z \cup \{x\})=k\}| > \\  |\{Z \cup \{y\} \in \mathcal{P}\mid \Rank_{\sqsupseteq}(Z \cup \{y\})=k\}|.
    \end{multline*}
    At level $k$, there are more sets containing $x$ than those containing $y$, i.\,e. $x_{k, \sqsupseteq}> y_{k, \sqsupseteq}$ by Definition~10. 
    To prove $x\succ^{\text{lex-cel}}_{\sqsupseteq}y$ it remains to show that $x_{i, \sqsupseteq}= y_{i, \sqsupseteq}$ for all $i<k$. By construction, for all $i<k$ and $Z\in \mathcal{P}\setminus\{x,y\}$, we know that $\Rank_{\sqsupseteq}(Z \cup \{x\})=\Rank_{\sqsupseteq}(Z \cup \{y\})$. Hence, for each set containing $x$ there is exactly one set containing $y$. By Definition~10, we obtain $x_{i, \sqsupseteq}= y_{i, \sqsupseteq}$, as desired.
\end{proof}

\section{Related Work}
In the following, let $A$ be an arbitrary set of objects and $\sqsupseteq$ is a preorder on the powerset $\mathcal{P}(A)$.

A prominent social ranking function is the \emph{Ceteris Paribus Majority Solution} (CP), which is defined in~\cite{DBLP:conf/ijcai/HaretKMO18} as follows.

\setcounter{definition}{22}
\begin{definition}
    For the preorder $\sqsupseteq$ and for any $x,y\in A$, we have that $x \succeq^{CP_\sqsupseteq} y$ if and only if
    \begin{multline*}
    |\{S \in \mathcal{P}(A \setminus \{x,y\}) | S \cup \{x\} \sqsupset S \cup \{y\}\}| \geq\\
    |\{S \in \mathcal{P}(A \setminus \{x,y\}) | S \cup \{y\} \sqsupset S \cup \{x\}\}|
    \end{multline*}
\end{definition}

Another relevant social ranking function is the \emph{Ordinal Banzhaf Index Solution} (BI)~\cite{DBLP:conf/ijcai/KhaniMO19}.
For that, we denote with $U_i = \{S\in \mathcal{P} \mid i\notin S\}$ the set of subsets that do not contain $i$ and with $U_{ij} = \{S \in \mathcal{P} \mid i,j \notin S\}$ the set of subsets that contain neither $i$ nor $j$.

First, we define the notion of \emph{ordinal marginal contribution} as follows.

\begin{definition}
    Let $\sqsupseteq$ be a preorder on $\mathcal{P}(A)$.
    The \emph{ordinal marginal contribution} $m_i^S(\sqsupseteq)$ of element $i$ wrt. the set $S$ with $i\notin S$, for the preorder $\sqsupseteq$ is defined as:
    \begin{align}
        m_i^S(\sqsupseteq) =
            \left\{\begin{array}{rl}
                 1 & \text{if}~S\cup\{i\} \sqsupset S,\\
                 -1 & \text{if}~S \sqsupset S\cup\{i\},\\
                 0 & \text{otherwise}.
            \end{array}\right.
    \end{align}
\end{definition}

We denote with $u_i^{+,\sqsupseteq}$ ($u_i^{-,\sqsupseteq}$) the set of subsets $S \in U_i$ such that $m_i^S(\sqsupseteq) = 1$ ($m_i^S(\sqsupseteq) = -1$) respectively.
Furthermore, we refer to the difference $s_i^\sqsupseteq = u_i^{+,\sqsupseteq} - u_i^{-,\sqsupseteq}$ as the \emph{ordinal Banzhaf score} of $i$ wrt. $\sqsupseteq$.

Finally, we define the social ranking solution based on the ordinal Banzhaf score as follows.

\begin{definition}
    For the preorder $\sqsupseteq$ and for any $x,y\in A$, we define that $x \succeq^{BI_{\sqsupseteq}} y$ if and only if
    \begin{align*}
        s_i^\sqsupseteq \geq s_j^\sqsupseteq
    \end{align*}
\end{definition}

\setcounter{figure}{2}
\begin{figure}[t] 
 \begin{center}
     \begin{tikzpicture}
        \node (a) [circle, draw] at (0, 0) {$a$};
        \node (b) [circle, draw] at (2, 0) {$b$};
        \node (c) [circle, draw] at (4, 0) {$c$};

        \draw [->,bend left] (a) to (b);
        \draw [->,bend left] (b) to (a);
        \draw [->, loop] (c) to (c);
        
    \end{tikzpicture}
     \end{center} 
 \caption{AF $F_3$ from Example \ref{ex:cp_sc}.}
  \label{fig:ex_cp_sc} 
\end{figure}
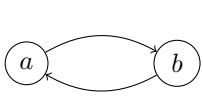

The social ranking argument-ranking semantics based on the two social ranking functions introduced above violate SC, as shown by the following examples.

\setcounter{example}{7}

\begin{example}\label{ex:cp_sc}
    The argument ranking $\succeq^{\text{CP}_\tau}$ violates SC for $\tau \in \{r\text{-}ad, r\text{-}co, r\text{-}gr, r\text{-}pr, r\text{-}sst\}$.
    Consider the AF $F_3$ in Figure~\ref{fig:ex_cp_sc}. 
    Then we have that $c \succeq^{\text{CP}_\tau} a$, which contradicts SC. 
\end{example}

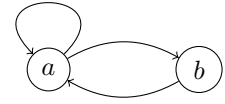
\begin{figure}[t] 
 \begin{center}
     \begin{tikzpicture}
        \node (a) [circle, draw] at (0, 0) {$a$};
        \node (b) [circle, draw] at (2, 0) {$b$};

        \draw [->,bend left] (a) to (b);
        \draw [->,bend left] (b) to (a);
        \draw [->, loop] (a) to (a);
        
    \end{tikzpicture}
     \end{center} 
 \caption{AF $F_4$ from Example \ref{ex:bi_sc}.}
  \label{fig:ex_bi_sc} 
\end{figure}

\begin{example}\label{ex:bi_sc}
    The argument ranking $\succeq^{\text{BI}_\tau}$ violates SC for $\tau \in \{r\text{-}ad, r\text{-}co, r\text{-}gr, r\text{-}pr, r\text{-}sst\}$.
    Consider the AF $F_4$ in Figure~\ref{fig:ex_bi_sc}.
    Then we have that $a \succeq^{\text{BI}_\tau}_F b$, which contradicts SC.
\end{example}

\end{document}